\documentclass{spie_fickus}

\usepackage{amsthm}
\usepackage{amssymb}
\usepackage{amsmath}
\usepackage{graphicx}
\usepackage{url}

\newtheorem{theorem}{Theorem}
\theoremstyle{definition}
\newtheorem{proposition}[theorem]{Proposition}

\newtheorem{problem}[theorem]{Problem}

\newtheorem{definition}[theorem]{Definition}

\title{Learning Boolean functions with concentrated spectra} 

\author{Dustin G.\ Mixon and Jesse Peterson
\skiplinehalf
Department of Mathematics and Statistics, Air Force Institute of Technology\\Wright-Patterson Air Force Base, Ohio 45433, USA
}
\authorinfo{Send correspondence to Dustin G.\ Mixon: E-mail: Dustin.Mixon@afit.edu}
 
\begin{document} 
\maketitle

\begin{abstract}
This paper discusses the theory and application of learning Boolean functions that are concentrated in the Fourier domain.
We first estimate the VC dimension of this function class in order to establish a small sample complexity of learning in this case.
Next, we propose a computationally efficient method of empirical risk minimization, and we apply this method to the MNIST database of handwritten digits.
These results demonstrate the effectiveness of our model for modern classification tasks.
We conclude with a list of open problems for future investigation.
\end{abstract}


\section{Introduction}

The last year has produced several breakthroughs in classification.
Deep neural networks now match human-level performance in both facial recognition\cite{TaigmanYRW:14} and general image recognition\cite{HeZRS:15}, and they also outperform both Apple and Google's proprietary speech recognizers\cite{HannunEtal:14}.
These advances leave one wondering about their implications for the general field of machine learning.
Indeed, Google and Facebook have actively acquired top talent in deep neural networks to pursue these leads\cite{Efrati:14,Hernandez:14}.
As a parallel pursuit, many have sought theoretical justification for the unreasonable effectiveness of deep neural networks\cite{AndenM:14,BrunaSL:14,MontufarPCB:14,ChoromanskaHMBL:15}.
Unfortunately, the theory is still underdeveloped, as we currently lack a theoretical grasp of the computational complexity of learning with modern deep neural networks.
Answers to such fundamental questions will help illustrate the scope of these emerging capabilities.

Observe that neural networks resemble circuit implementations of Boolean functions.
Indeed, a circuit amounts to a directed acyclic graph with $n$ input nodes and a single output node, along with intermediate nodes that represent Boolean logic gates (such as ANDs, ORs, threshold gates, etc.).
As such, one may view circuits as discrete analogies for neural networks.
To date, there is quite a bit of theory behind the learnability of Boolean functions with sufficiently simple circuit implementations\cite{LinialMN:93,JacksonKS:02,ShpilkaTV:13}.
The main idea is that such functions enjoy a highly concentrated Fourier transform due to a clever application of Hastad's Switching Lemma\cite{Hastad:87}.
Passing through the analogy, one might then hypothesize that the real-world functions that are well approximated by learnable deep neural networks also enjoy a highly concentrated Fourier transform---this hypothesis motivates our approach.

This paper discusses the theory and application of learning Boolean functions that are concentrated in the Fourier domain.
The following section provides some background material on statistical learning theory to help set the stage for our investigation.
We then prove in Section~3 that the sample complexity of learning Boolean functions of concentrated spectra is small.
Section~4 proposes a learning algorithm as a first step towards tackling the computational complexity, and then Section~5 illustrates how well our model performs on a real-world dataset (namely, the MNIST database of handwritten digits\cite{LeCunCB:15}).
We conclude in Section~6 with a list of open problems.

\section{Background}

The objective is to estimate an unknown labeling function $f\colon\{\pm1\}^n\rightarrow\{\pm1\}$.
A sample $\{x_i\}_{i=1}^\ell\subseteq\{\pm1\}^n$ is drawn i.i.d.\ according to some unknown distribution $p$, and we receive the labeled training set $\{(x_i,f(x_i))\}_{i=1}^\ell$.
The quality of our estimate $\hat{f}\colon\{\pm1\}^n\rightarrow\{\pm1\}$ will be evaluated in terms of the \textbf{risk functional}
\[
R(\hat{f},f)
:=\sum_{x\in\{\pm1\}^n}1_{\{\hat{f}(x)\neq f(x)\}}p(x).
\]
Risk is commonly approximated by \textbf{empirical risk} using a random sample $\{y_i\}_{i=1}^m$ drawn i.i.d.\ according to $p$:
\[
R_{y}(\hat{f},f)
:=\frac{1}{m}\sum_{i=1}^m 1_{\{\hat{f}(y_i)\neq f(y_i)\}}.
\]
In practice, empirical risk is used to evaluate an estimate $\hat{f}$ with the help of a labeled test set (which is disjoint from the training set).
Similarly, it makes sense to pick $\hat{f}$ in such a way that minimizes empirical risk $R_x$ over the training set.
However, the training sample only covers a small fraction of the sample space $\{\pm1\}^n$, so how do we decide which values $\hat{f}$ should take beyond this sample?

The trick is to restrict our empirical risk minimization to a ``simple'' function class $C\subseteq\{g\colon\{\pm1\}^n\rightarrow\{\pm1\}\}$.
Intuitively, if we suspect that $f$ is close to some member $\hat{f}$ of a small set $C$, and a large training set happens to nearly match one such member, then chances are small that this occurred by mere coincidence, and so we should expect $R(\hat{f},f)$ to be small.
In general, $C$ doesn't need to be small, as it suffices for $C$ to enjoy a broader notion of simplicity:

\begin{definition}
The function class $C\subseteq\{g\colon\{\pm1\}^n\rightarrow\{\pm1\}\}$ is said to \textbf{shatter} $\{x_i\}_{i=1}^\ell\subseteq\{\pm1\}^n$ if for every choice of labels $y_1,\ldots,y_\ell\in\{\pm1\}$, there exists a function $g\in C$ such that $g(x_i)=y_i$.
The \textbf{VC dimension} of $C$ is the size of the largest set that $C$ shatters.
\end{definition}

We consider $C$ to be simple if its VC dimension is small.
By counting, it is clear that the VC dimension of $C$ is $\leq\log_2|C|$, meaning small sets are necessarily simple.
The following result illustrates the utility of VC dimension as a notion of simplicity:

\begin{theorem}[obtained by combining equations (3.15) and (3.23) from Vapnik\cite{Vapnik:00}]
\label{thm.VC}
Fix $C\subseteq\{g\colon\{\pm1\}^n\rightarrow\{\pm1\}\}$ and let $h$ denote its VC dimension.
Pick $f\colon\{\pm1\}^n\rightarrow\{\pm1\}$ and draw $\{x_i\}_{i=1}^\ell\subseteq\{\pm1\}^n$ i.i.d.\ according to some (unknown) distribution $p$.
Then with probability $\geq1-\eta$,
\[
R(g,f)
\leq R_x(g,f)+\sqrt{\frac{h(\log(2\ell/h)+1)-\log(\eta/4)}{\ell}}
\]
for all $g\in C$ simultaneously.
Here, the probability is on $\{x_i\}_{i=1}^\ell$.
\end{theorem}

In words, Theorem~\ref{thm.VC} states that the risk of an estimate is small provided its empirical risk is small over a sufficiently large training set (namely, $\ell\gg h$).
This suggests three properties that we want our function class $C$ to satisfy:
\begin{itemize}
\item
\textbf{Simple.}
We want the VC dimension of $C$ to be small, so as to allow for a small sample complexity.
\item
\textbf{Admits fast optimization.}
We want empirical risk minimization over $C$ to be computationally efficient.
\item
\textbf{Models reality.}
Given an application, we want the true function to be close to some member of $C$.
\end{itemize}
In the remainder of this paper, we study whether Boolean functions with concentrated spectra form a function class which satisfies these desiderata.
To be explicit, the following defines the function class of interest:

\begin{definition}
Let $C_{n,k}$ denote the class of all functions $g\colon\{\pm1\}^n\rightarrow\{\pm1\}$ for which there exist index sets $S_1,\ldots,S_k\subseteq[n]$ and coefficients $a_1,\ldots,a_k\in\mathbb{R}$ such that
\begin{equation}
\label{eq.Cnk}
g(x)
=\operatorname{sign}\Bigg(\sum_{i=1}^ka_i\prod_{j\in S_i}x_j\Bigg)
\qquad
\forall x\in\{\pm1\}^n.
\end{equation}
\end{definition}

In the following section, we estimate $C_{n,k}$'s VC dimension.
Next, Section~4 proposes a method for performing empirical risk minimization over $C_{n,k}$.
Finally, we apply this method to the MNIST database of handwritten digits\cite{LeCunCB:15} in Section~5 to illustrate how well $C_{n,k}$ models reality (at least in one application).

\section{Estimating the VC dimension}

Considering Theorem~\ref{thm.VC}, we desire a function class of small VC dimension, as this will allow us to get away with a small training set.
In this section, we estimate the VC dimension of $C_{n,k}$, the class of functions of the form \eqref{eq.Cnk}.
To this end, it is helpful to identify how \eqref{eq.Cnk} is related to the Walsh--Hadamard transform $W\colon\ell_2(\mathbb{Z}_2^n)\rightarrow\ell_2(\mathbb{Z}_2^n)$, defined by
\[
(Wz)(v)
:=\sum_{u\in\mathbb{Z}_2^n}z(u)(-1)^{\sum_{j=1}^nu_jv_j}
\qquad
\forall v\in\mathbb{Z}_2^n.
\]
Taking $S_u:=\{j:u_j=1\}$ and $x_j:=(-1)^{v_j}\in\{\pm1\}$, we equivalently have
\[
(Wz)(\log_{-1}(x))
=\sum_{u\in\mathbb{Z}_2^n}z(u)\prod_{j\in S_u}x_j
\qquad
\forall x\in\{\pm1\}^n.
\]
Note that every real polynomial over $\{\pm1\}^n$ is of the form $(Wz)\circ\log_{-1}$.
When the coefficients are $k$-sparse, taking the sign of this polynomial produces a member of $C_{n,k}$.

Recall the matrix representation of $W$, namely
\[
W=\left[\begin{array}{rr}1&1\\1&-1\end{array}\right]^{\otimes n}.
\]
The vector $Wz$ lists all $2^n$ possible outputs of $(Wz)\circ\log_{-1}$.
As such, we may identify $C_{n,k}$ with 
\begin{equation}
\label{eq.identify}
\{\operatorname{sign}(Wz):\|z\|_0\leq k\}.
\end{equation}
We use this identification to prove the main result of this section:

\begin{theorem}
The VC dimension of $C_{n,k}$ is $\leq 2nk-O(k\log k)$.
\end{theorem}

\begin{proof}
Since the VC dimension is $\leq\log_2|C_{n,k}|$, it suffices to estimate $|C_{n,k}|$.
Considering \eqref{eq.identify}, this quantity is the number of orthants in $\mathbb{R}^{2^n}$ that intersect the union of subspaces $\{Wz:\|z\|_0\leq k\}$.
By Theorem~V.1 of van der Berg and Friedlander\cite{vanderBergF:10}, each subspace intersects at most $2\binom{2^n-1}{\leq k-1}$ orthants, and so
\[
|C_{n,k}|
\leq\binom{2^n}{k}\cdot 2\binom{2^n-1}{\leq k-1}
\leq\binom{2^n}{k}\cdot 2\binom{2^n-1+k-1}{k-1}
\leq\binom{2^{n+1}}{k}^2.
\]
Taking logs of both sides and applying the bound $\binom{a}{b}\leq(\frac{ea}{b})^b$ then gives the result.
\end{proof}

As a complementary result, the following illustrates how tight our bound is:

\begin{proposition}
The VC dimension of $C_{n,k}$ is $\geq \max\{n,k\}$.
\end{proposition}

\begin{proof}
It suffices to find a subcollection $S$ of $\max\{n,k\}$ row indices such that
\[
\{\operatorname{sign}(W_Sz):\|z\|_0\leq k\}=\{\pm1\}^{|S|},
\]
where $W_S$ denotes the $|S|\times 2^n$ submatrix of rows from $W$ indexed by $S$.
In the case where $n\geq k$, let $S$ index the rows of $W$ which have the form
\[
w_i:=(1,1)^{\otimes (i-1)}\otimes(1,-1)\otimes(1,1)^{\otimes(n-i)}
\]
for some $i\in[n]$.
Note that for every $u\in\mathbb{Z}_2^n$, the corresponding identity basis element is given by $e_u:=\bigotimes_{i=1}^n e_{u_i}$, where $e_0:=(1,0)$ and $e_1:=(0,1)$.
This then implies
\[
(W_Se_u)_i
=\langle e_u,w_i\rangle
=\langle e_{u_i},(1,-1)\rangle
=(-1)^{u_i}.
\]
As such, every member of $\{\pm1\}^n$ is can be expressed as $W_ne_u$ for some identity basis element $e_u$.
In the remaining case where $k\geq n$, observe that $W$ has rank $2^n\geq k$, and so there exists a $k\times k$ submatrix $A$ of rank $k$.
Let $S$ denote the row indices of $A$.
Then every vector in $\{\pm1\}^k$ can be expressed as $W_Sz$, where $z$ is supported on the column indices of $A$.
\end{proof}

In pursuit of better lower bounds, we need more techniques to tackle the notion of shattering.
The proof of the following proposition provides some ideas along these lines:

\begin{proposition}
Given an $N\times N$ matrix $A$, denote $C_{A,k}:=\{\operatorname{sign}(Az):\|z\|_0\leq k\}$.
\begin{itemize}
\item[(a)]
$C_{A,k}$ fails to shatter a set of size $O(k^2)$ if $A$ is the Walsh--Hadamard transform matrix and $k\ll\sqrt{N}$.
\item[(b)]
$C_{A,k}$ fails to shatter a set of size $O(k\log(N/k))$ w.h.p.\ if the entries of $A$ are i.i.d.\ $\mathcal{N}(0,1)$.
\end{itemize}
\end{proposition}

\begin{proof}
(a)
Define the sign rank of a matrix $S$ with entries in $\{\pm1\}$ to be the minimum rank of all matrices $M$ satisfying $M_{ij}S_{ij}>0$ for every $i,j$. 
Write $N=2^n$, and observe that each column of $A$ can be reshaped to be a $2^{\lfloor n/2\rfloor}\times 2^{\lceil n/2\rceil}$ matrix of rank~$1$.
Then reshaping $\operatorname{sign}(Az)$ in the same way produces a $2^{\lfloor n/2\rfloor}\times 2^{\lceil n/2\rceil}$ matrix of sign rank $\leq k$.
By the left-hand inequality of equation (1) from Alon, Frankl and R\"{o}dl\cite{AlonFR:85}, for every $m$, there exists an $m\times m$ matrix $S$ with entries in $\{\pm1\}$ of sign rank $\geq m/32$.
Taking $m:=32(k+1)\leq 2^{\lfloor n/2\rfloor}$, use the corresponding matrix $S$ to construct a $2^{\lfloor n/2\rfloor}\times 2^{\lceil n/2\rceil}$ matrix $S'$ by padding with $\pm1$s.
Then $S'$ has sign rank $\geq k+1$ (implying $S'\not\in C_{A,k}$) regardless of how the padded entries are selected.
As such, $C_{A,k}$ fails to shatter these $m^2$ entries.

(b)
We seek $m$ such that for every $f\in C_{A,k}$, the first $m$ entries of $f$ are not all $1$s.
Let $A_{I,J}$ denote the submatrix with row indices in $I$ and column indices in $J$.
Then equivalently, we seek $m$ such that for every $K\subseteq[N]$ with $|K|=k$, the subspace $\operatorname{im}(A_{[m],K})$ intersects the nonnegative orthant $\mathbb{R}_{\geq0}^m$ uniquely at the origin.
Since $\operatorname{im}(A_{[m],K})$ is drawn uniformly from the Grassmannian of $k$-dimensional subspaces in $\mathbb{R}^m$, we may apply Gordon's Escape Through a Mesh Theorem (namely, Corollary~3.4 of Gordon\cite{Gordon:88}).
For this theorem, we use the fact that the Gaussian width of the positive orthant is $\sqrt{m/2}-O(1/\sqrt{m})$, as established by Propositions~3.2 and~10.2 in Amelunxen et al.\cite{AmelunxenLMT:14}.
Then
\[
\operatorname{Pr}\Big(\operatorname{im}(A_{[m],K})\cap\mathbb{R}_{\geq0}^m=\{0\}\Big)\geq1-\frac{7}{2}\operatorname{exp}\bigg(-\frac{1}{18}\bigg(\sqrt{m-k}-\sqrt{\frac{m}{2}}-O\bigg(\frac{1}{\sqrt{m-k}}+\frac{1}{\sqrt{k}}\bigg)\bigg)^2\bigg).
\]
Taking $m=ck\log(N/k)$ for sufficiently large $c$, the union bound then gives
\[
\operatorname{Pr}\Big(\operatorname{im}(A_{[m],K})\cap\mathbb{R}_{\geq0}^m=\{0\}\quad \forall K\subseteq[N], |K|=k \Big)\geq1-e^{-\Omega(k\log(N/k))}.
\]
As such, $C_{A,k}$ fails to shatter the first $m$ entries.
\end{proof}

\section{Empirical risk minimization}
\label{section.erm}

In this section, we consider the problem of empirical risk minimization over $C_{n,k}$:

\begin{problem}
\label{prob.approx ratio}
Let $W$ denote the $2^n\times 2^n$ Walsh--Hadamard transform matrix, and let $z$ be some (nearly) $k$-sparse vector in $\mathbb{R}^{2^n}$.
Given a sample $x$ of $\ell$ entries of $f=\operatorname{sign}(Wz)$, find $\hat{f}\in C_{n,k}$ satisfying
\begin{equation}
\label{eq.approx ratio}
R_x(\hat{f},f)
\leq \operatorname{const}\cdot\min_{g\in C_{n,k}}R_x(g,f)
\end{equation}
in $\operatorname{poly}(n,k,\ell)$ time.
\end{problem}

This problem can be viewed as a combination of the sparse fast Walsh--Hadamard transform\cite{LiBPR:14} and one-bit compressed sensing\cite{JacquesLBB:13}.
Computationally, the main difficulty seems to be achieving polynomial time in $n$ rather than $2^n$, and to do so, one must somehow take advantage of the rich structure provided by the Walsh--Hadamard transform.
As a cheap alternative, this paper instead simplifies the problem by strengthening the assumptions, namely, that $z$ is mostly supported on members of $\mathbb{Z}_2^n$ with Hamming weight $\leq d$, that is, the polynomial $(Wz)\circ\log_{-1}$ has degree at most $d$.
This allows us to discard the vast majority of columns of $W$, leaving only $O(n^d)$ columns to consider.

Let $W_{x,d}$ denote the submatrix of $W$ with row indices in the sample $x$ and column indices of Hamming weight $\leq d$.
Also, let $f_x$ denote the true function $f$ restricted to the sample $x$.
To isolate an estimate $\hat{f}\in C_{n,k}$, we first perform feature selection by finding the columns of $W_{x,d}$ which look the most like $f_x$.
That is, we find the largest entries of $|W_{x,d}^\top f_x|$ and isolate the corresponding columns of $W_{x,d}$.
Thanks to the discrete nature of $W$, some of these columns may be identical up to a global sign factor.
As such, we collect columns of $W_{x,d}$ corresponding the largest entries of $|W_{x,d}^\top f_x|$ until we have $k$ distinct columns up to sign.
Let $A$ denote the resulting $\ell\times k$ matrix of columns.
Then it remains to find a coefficient vector $z$ such that $f_x\approx\operatorname{sign}(Az)$, that is, to train a support vector machine.
After finding $z$, we pick $\hat{f}\in C_{n,k}$ according to \eqref{eq.Cnk} by taking the coefficients $\{a_i\}_{i=1}^k$ to be the entries of $z$ and taking the index sets to be $S_i=\{j:u_j=1\}$, where $u\in\mathbb{Z}_2^n$ is the column index of $W$ corresponding to the $i$th column of $A$.

We note that alternatively, one could train an $\ell_1$-restricted support vector machine\cite{ZhuRHT:04} to find a sparse $z$ such that $f_x\approx\operatorname{sign}(W_{x,d}z)$:
\begin{align}
\label{eq.SVM}
\mbox{min}&\quad\sum_{i=1}^\ell \Big(1-(f_x)_i(W_{x,d}z)_i\Big)_+\\
\nonumber
\mbox{s.t.}&\quad\|z\|_1\leq \tau.
\end{align}
However, we found this to be slow in practice, even for small values of $d$.
Still, we followed the intent of this method by applying a loose $\ell_1$ restriction to our support vector machine training, albeit after we performed feature selection.

We note that for a fixed $d$, our two-step method (feature selection, then support vector machine training) runs in time which is polynomial in $n$, $k$ and $\ell$.
Unfortunately, we currently lack a performance guarantee of the form \eqref{eq.approx ratio}.
Instead, we apply this method to real-world data in the following section to illustrate its effectiveness (as well as the quality of the function model $C_{n,k}$).

\section{Implementation with handwritten digits}

The MNIST database of handwritten digits\cite{LeCunCB:15} contains $5923$ zeros and $6742$ ones, and each digit image is represented by a $28 \times 28$ matrix with entries in $[0,1]$.
To keep runtime reasonable, we reduced the image to a $5 \times 5$ matrix by convolving with the indicator function of a $5 \times 5$ block and sampling over a $5 \times 5$ grid.
We then thresholded the entries to obtain vectors in $\{\pm 1\}^{25}$; typical results of this process are illustrated in Figure~\ref{fig.MNIST}.
A $-1$ label was assigned to the zeros, and ones were similarly labeled with a $1$.
At this point, classification amounts to learning a function $f\colon\{\pm1\}^{25}\rightarrow\{\pm1\}$.

After processing the data in this way, we implemented our method of feature selection and support vector machine training as detailed in Section~\ref{section.erm}.
We assumed the polynomial $f=(Wz)\circ \log_{-1}$ has degree at most $d$, and we fix $d=3$.
We chose this value because we found that increasing $d$ greatly increases runtime without empirically improving the classifier.
In order to choose a sparsity level $k$, we performed feature selection and trained a support vector machine for multiple choices of $k$.
Intuitively, taking $k$ too small will overly simplify the model and fail to match the inherent complexities of the data, while large choices for $k$ will lead to overfitting.
Before picking $k$, we first let this parameter range from $10$ to $280$ in increments of $10$, and for each value of $k$, we performed $10$ experiments in which we chose disjoint training and testing sets of sizes $1500$ and $2500$ respectively from each label class.
These sets were chosen uniformly at random without replacement.
We decided to make the training set small relative to the entire database due to the long runtime required to train a support vector machine---we used a much larger training set after we identified the ``optimal'' $k$.

\begin{figure}
\begin{center}
\includegraphics{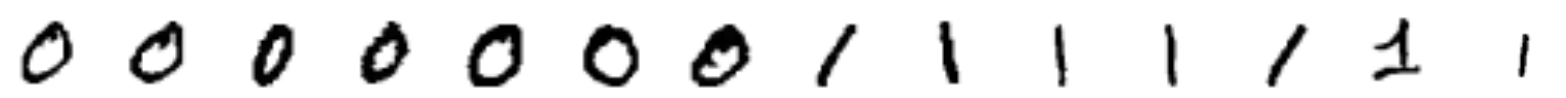}\\
\includegraphics{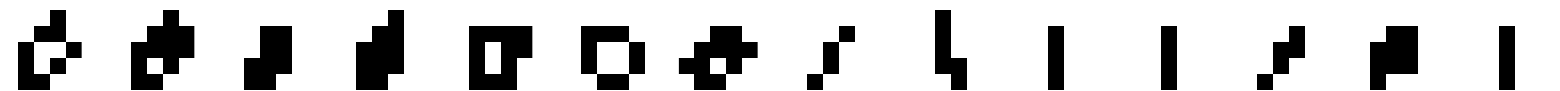}
\end{center}
\caption{
Sample of MNIST database of handwritten digits\cite{LeCunCB:15}.
(top row)
Each image is a $28\times 28$ matrix with entries in $[0,1]$.
(bottom row)
In order to minimize the runtime of learning, we decided to convert each image into a $5\times 5$ image by convolving with the indicator function of a $5\times 5$ block and sampling over a $5\times 5$ grid.
Finally, we thresholded the entries to produce a vector in $\{\pm1\}^{25}$.
Identifying the label \texttt{zero} with $-1$ and \texttt{one} with $1$, the classification task amounts to learning a function $f\colon\{\pm1\}^{25}\rightarrow\{\pm1\}$, and so we apply the method described in Section~\ref{section.erm}.
}
\label{fig.MNIST}
\end{figure}

\begin{figure}
\begin{center}
\includegraphics[height=0.35\textwidth]{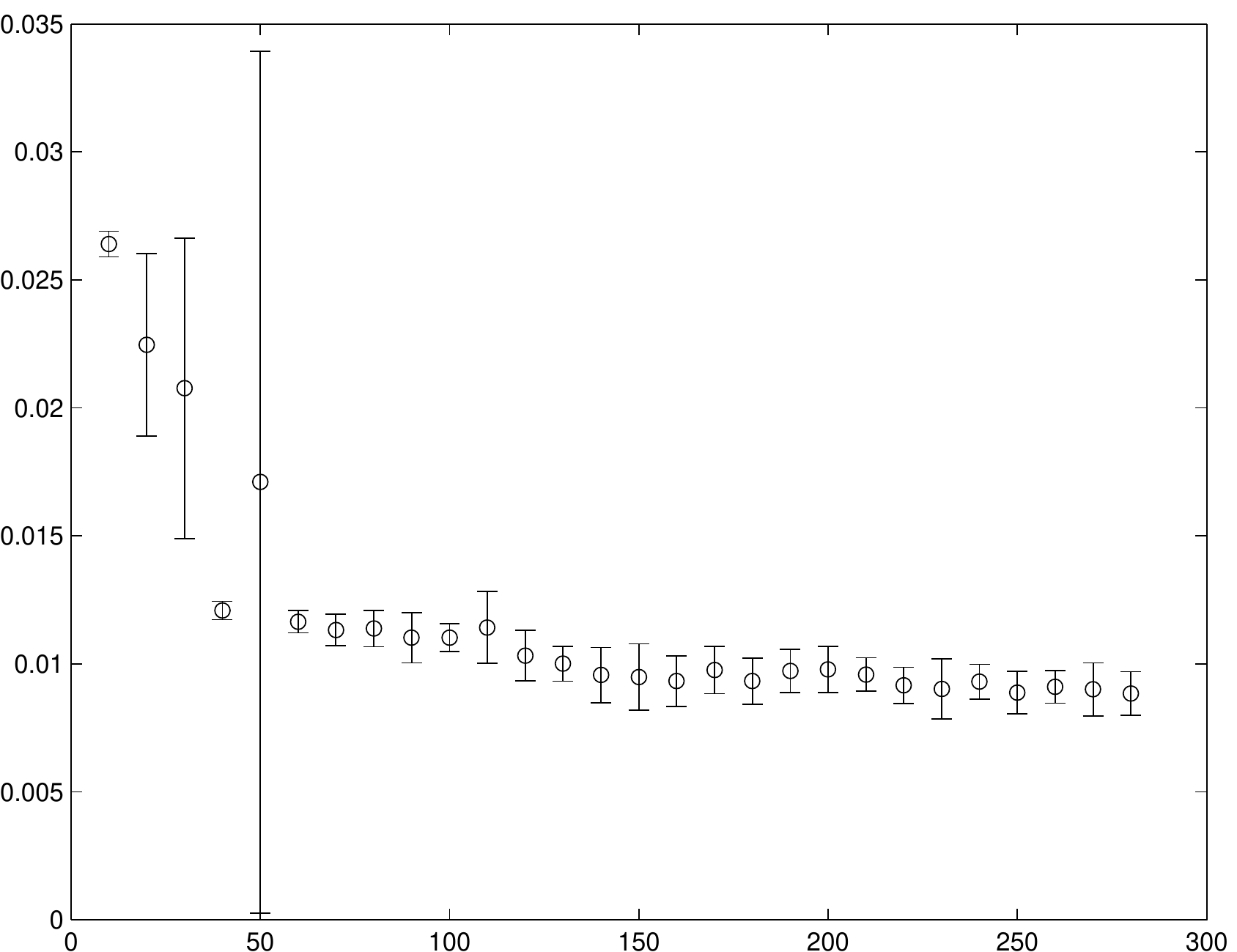}
\quad
\includegraphics[height=0.35\textwidth]{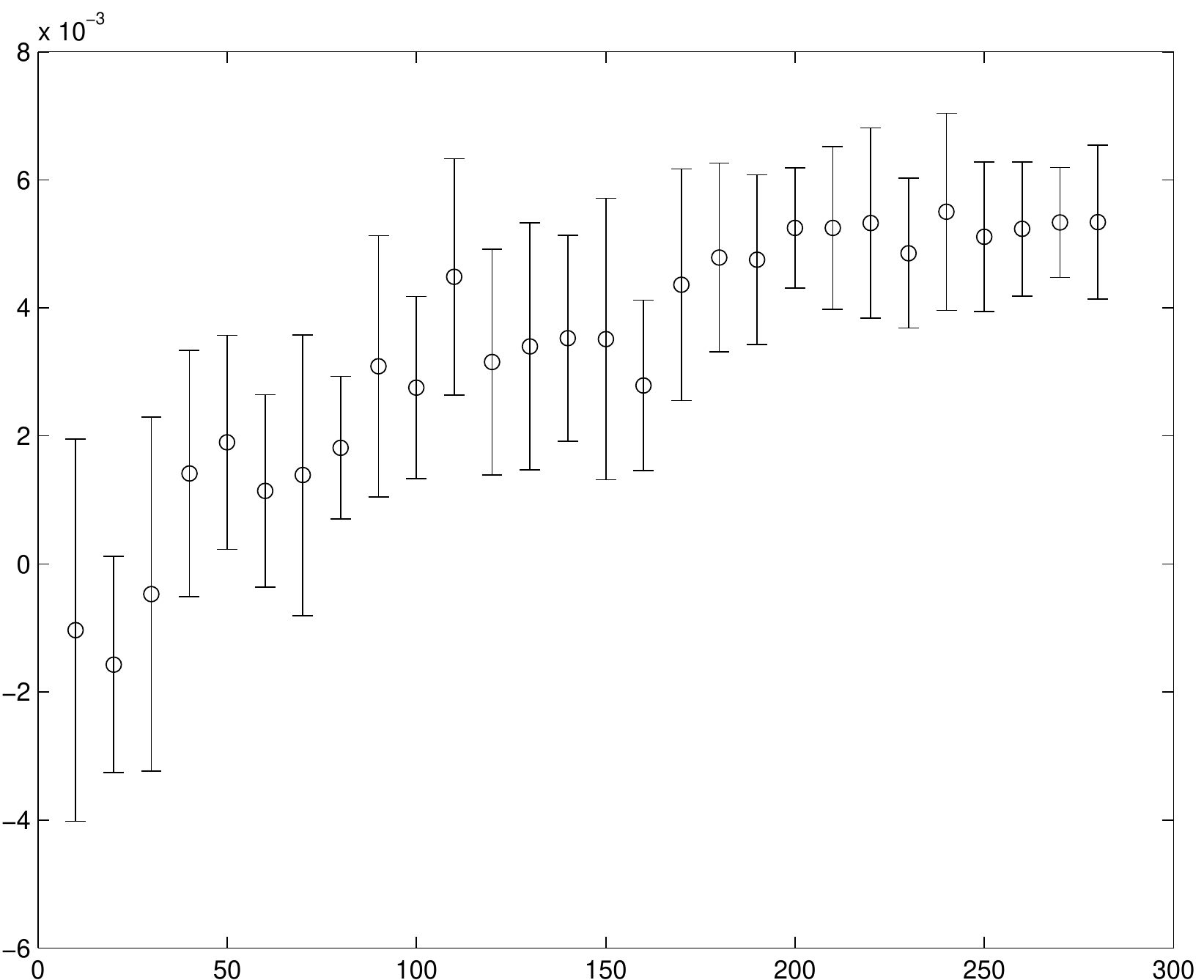}
\end{center}
\caption{
(left)
For each $k=10:10:280$, we ran $10$ trials of the following experiment:
Draw a random training set of $1500$ zeros and $1500$ ones, train a classifier with $d=3$ and parameter $k$ according to Section~\ref{section.erm}, draw a random test set of $2500$ zeros and $2500$ ones (disjoint from the training set), and record the misclassification rate.
Here, we plot the average rate along with errors bars indicating one standard deviation above and below.
For $k=50$, the error bars are large due to a single outlier.
The misclassification rate plateaus after around $k=150$, and so to minimize computational complexity, we selected $k=150$ for our final classifier.
For larger values of $k$, we expect the misclassification rate to increase due to overfitting, but long runtimes prevented us from performing such experiments.
(right)
For each trial of our experiment, we also recorded the misclassification rate in the training set.
The difference between the test and training misclassification rates forms a proxy for $R(\hat{f},f)-R_x(\hat{f},f)$.
Qualitatively, the plot of these differences matches the behavior predicted by the square-root term in Theorem~\ref{thm.VC}, but the values in
the plot are orders of magnitude smaller.
This suggests that the guarantee provided by the theorem is conservative in our setting.
}
\label{fig.error}
\end{figure}

\begin{figure}
\begin{center}
\includegraphics{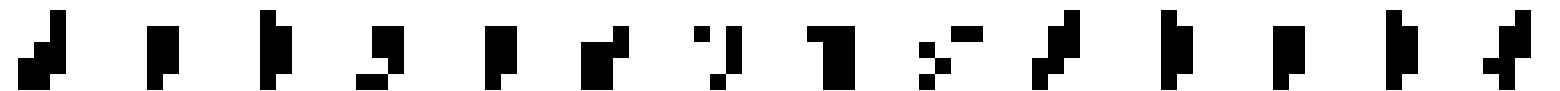}\\
\includegraphics{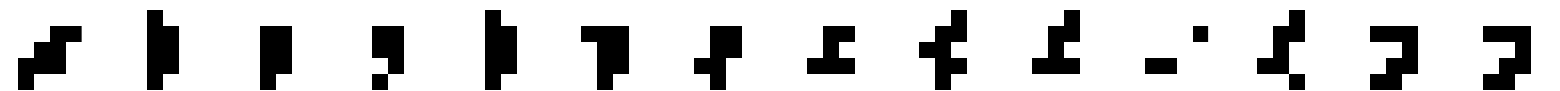}
\end{center}
\caption{
All 28 digits that were misclassified by our classifier.
Can you guess which are which?
We suspect the classifier would perform better if it received the original images instead of downsampled versions, but this would be computationally costly.
Indeed, we believe new algorithms need to be developed before our techniques can compete with the state of the art.
Spoiler: The last 7 digits are ones, and the rest are zeros.}
\label{fig.misclass}
\end{figure}

Given a training sample, we performed feature selection as described in Section~\ref{section.erm} and trained a support vector machine as is in \eqref{eq.SVM} with a loose $\ell_1$ restriction ($\tau = 1000$).
We then calculated the empirical risk using the test set and plotted the results in the left portion of Figure~\ref{fig.error}.
This figure shows the sparsity level $k$ versus the mean empirical risk with error bars denoting one standard deviation from the mean.
The large error bars for $k=50$ are due to a single extreme outlier.
Observe that empirical risk decreases and plateaus at around $k=150$, and while we would expect the curve to trend upward for larger $k$ due to overfitting, we terminated our computations at $k=180$ due to computation time.

For each trial of our experiment, we also recorded the misclassification rate in the training set.
The difference between the test and training misclassification rates forms a proxy for $R(\hat{f},f)-R_x(\hat{f},f)$.
Qualitatively, the plot of these differences matches the behavior predicted by the square-root term in Theorem~\ref{thm.VC}, but the values in the plot are orders of magnitude smaller.
This suggests that the guarantee provided by the theorem is conservative in our setting.

Since empirical risk fails to noticeably decrease after about $k=150$, we selected this value for $k$ and performed our variable selection and training process on a large training set.
Specifically, we randomly chose a training set consisting of $4000$ zeros and $4000$ ones (approximately two thirds of the entire database).
Our test set then consisted of $1900$ zeros and $1900$ ones.
With these training and test sets, our choice of $k$ achieved a misclassification rate of $0.74\%$ after a total runtime of $160$ seconds.
Considering we greatly downsampled our data from the MNIST database and we only attempted to classify zeros and ones, there is no direct comparison to be made with existing results\cite{LeCunCB:15}.
Still, it is worth mentioning that the best SVM classifier exhibits a misclassification rate (on all digits $0$ through $9$) of $0.56\%$, suggesting that our results are reasonable.
To make this point stronger, Figure~\ref{fig.misclass} displays all $28$ misclassified digits from the test set.
We contend that a human would likely misclassify these digits as well.
Can you pick out the zeros from the ones? 

\section{Conclusion and open problems}

This paper demonstrates the plausibility of learning Boolean functions with concentrated spectra, as well as its applicability to modern classification theory and application.
However, this paper offers more questions than answers.
For example, while we showed in Section~2 that the VC dimension $h$ of $C_{n,k}$ satisfies
\[
\max\{n,k\}\leq h \leq 2nk-O(k\log k),
\]
we have yet to identify how $h$ scales with $n$ and $k$.
This leads to our first open problem:

\begin{problem}
Determine the VC dimension of $C_{n,k}$.
\end{problem}

In Section~3, we proposed an algorithm for empirical risk minimization over polynomials in $C_{n,k}$ of degree at most $d$.
However, we still don't know if this algorithm produces an estimate $\hat{f}$ that satisfies a guarantee of the form \eqref{eq.approx ratio}.

\begin{problem}
Find a performance guarantee for the algorithm proposed in Section~3.
\end{problem}

Problem~\ref{prob.approx ratio} also remains open, but before this can be solved, one must first devise a candidate algorithm.
This leads to the following intermediate problem:

\begin{problem}
\label{prob.realworld}
Let $W$ denote the $2^n\times 2^n$ Walsh--Hadamard transform matrix, and let $z$ be some (nearly) $k$-sparse vector in $\mathbb{R}^{2^n}$.
Given a sample $x$ of $\ell$ entries of $f=\operatorname{sign}(Wz)$, find $\hat{f}\in C_{n,k}$ that well approximates $f$ (at least empirically) in $\operatorname{poly}(n,k,\ell)$ time.
\end{problem}

This is perhaps the most important open problem in this paper.
Considering our implementation in Section~5, the algorithm proposed in Section~3 exhibits certain computational bottlenecks due to the poor dependence on $d$.
As such, the methods of this paper might fail to compete with state-of-the-art classification until we find a solution to Problem~\ref{prob.realworld}.

Finally, while Section~5 demonstrated the effectiveness of our model for handwritten digits, we have yet to determine the full scope of its applicability.
This suggests the need for more numerical experiments, but there is also a theoretical result to seek.
Indeed, our approach was motivated by a certain hypothesis, and the confirmation of this hypothesis remains an open problem:

\begin{problem}
Prove that Boolean functions that are well approximated by learnable deep neural networks also enjoy a highly concentrated Fourier transform.
\end{problem}

Such a result would establish that empirical risk minimization over $C_{n,k}$ amounts to a relaxation of the corresponding optimization over deep neural networks, and so our model would consequently inherit the real-world utility of such networks.

\section*{ACKNOWLEDGMENTS}
This work was supported by an AFOSR Young Investigator Research Program award, NSF Grant No.\ DMS-1321779, and AFOSR Grant No.\ F4FGA05076J002.
The views expressed in this article are those of the authors and do not reflect the official policy or position of the United States Air Force, Department of Defense, or the U.S.\ Government.

\end{document}